%% file: ms.tex
\documentclass[dvipsnames, table, xcdraw]{article} 
\usepackage{textcomp}

\usepackage{highlight}
\usepackage{subcaption}

\usepackage{enumitem}

\usepackage{sidecap}

\usepackage{wrapfig}

\usepackage{caption}

\usepackage{float}

\usepackage{url}

\usepackage{breakurl}
\usepackage[breaklinks]{hyperref}

\usepackage[ruled,vlined,linesnumbered]{algorithm2e}

\usepackage{iclr2022_conference,times}
\usepackage{amsmath,amssymb,amsfonts,amsthm}

\usepackage{xcolor}
\usepackage{tikz}
\usepackage{pgfplots}

\input{math_commands.tex}

\usepackage{hyperref}
\usepackage{cleveref}

\usepackage{hhline}

\usepackage{multirow}

\makeatletter \renewcommand\paragraph{\@startsection{paragraph}{4}{\z@}                                     {3mm}                                     {-1em}                                    {\normalfont\normalsize\bfseries}} \makeatother
 
\usepackage{url}

\usepackage[cal=boondoxo]{mathalfa}

\usepackage{accents}
\newlength{\dhatheight}
\newcommand{\doublehat}[1]{%
\settoheight{\dhatheight}{\ensuremath{\hat{#1}}}%
\addtolength{\dhatheight}{-0.35ex}%
\hat{\vphantom{\rule{1pt}{\dhatheight}}%
\smash{\hat{#1}}}}

\newtheorem{lemma}{Lemma}
\newtheorem{theorem}{Theorem}

\newcommand{\bracket}[1]{\left[#1\right]}
\renewcommand{\brace}[1]{\left(#1\right)}

\newcommand{\quantizer}[0]{\textrm{\fontfamily{cmtt}\selectfont{Q}}}

\newcommand{\Lagr}{\mathcal{L}}
\newcommand{\globloss}{G}
\newcommand{\locloss}{\hat{G}}
\newcommand{\raloss}{\doublehat{{G}}}
\newcommand{\textrect}[1]{\protect\tikz\protect\fill[#1](0,0)rectangle(1em,0.5em);}
\newcommand{\dataset}[1]{{\fontfamily{cmss}\selectfont#1}}

\SetKwFunction{FEncode}{Encode}
\SetKwFunction{FDecode}{Decode}
\SetKwFunction{FServer}{RunServer}
\SetKwFunction{FClient}{RunClient}
\SetKwFunction{FSend}{Send}
\SetKwFunction{FReceive}{Receive}
\SetKwFunction{FRound}{Round}
\SetKwProg{Fn}{Function}{}{end}

\pgfplotsset{
  /pgfplots/xlabel near ticks/.style={
     /pgfplots/every axis x label/.style={
        at={(ticklabel cs:0.5)},anchor=near ticklabel
     }
  },
  /pgfplots/ylabel near ticks/.style={
     /pgfplots/every axis y label/.style={
        at={(ticklabel cs:0.5)},rotate=0,anchor=near ticklabel}
     }
  }
  
\allowdisplaybreaks

\title{DAdaQuant: Doubly-adaptive quantization for communication-efficient Federated Learning}

\author{Robert Hönig, Yiren Zhao, Robert Mullins \\
Department of Computer Science\\
University of Cambridge\\
\texttt{\{rh723,yiren.zhao,robert.mullins\}@cl.cam.ac.uk} \\
}

\iclrfinalcopy 
\begin{document}



\maketitle

\begin{abstract}
Federated Learning (FL) is a powerful technique for training a model on a
server with data from several clients in a privacy-preserving manner. In FL,
a server sends the model to every client, who then train the model locally
and send it back to the server. The server aggregates the updated models and
repeats the process for several rounds. FL incurs significant communication
costs, in particular when transmitting the updated local models from the
clients back to the server. Recently proposed algorithms quantize the model
parameters to efficiently compress FL communication. These algorithms
typically have a quantization level that controls the compression factor. We
find that dynamic adaptations of the quantization level can boost
compression without sacrificing model quality. First, we introduce a
time-adaptive quantization algorithm that increases the quantization level
as training  progresses. Second, we introduce a client-adaptive quantization
algorithm that assigns each individual client the optimal quantization level
at every round. Finally, we combine both algorithms into DAdaQuant, the
doubly-adaptive quantization algorithm. Our experiments show that DAdaQuant
consistently improves client$\rightarrow$server compression, outperforming
the strongest non-adaptive baselines by up to $2.8\times$.

\end{abstract}

\section{Introduction}  

Edge devices such as smartphones, remote sensors and smart home appliances
generate massive amounts of data \citep{wang2018smart, cao2017deepmood, shi2016promise}. In recent years, Federated Learning (FL)
has emerged as a technique to train models on this data while
preserving privacy \citep{FedAvg,FedProx}.

In FL, we have a single server that is connected to many clients. Each
client stores a local dataset that it does not want to share with the server
because of privacy concerns or law enforcement \citep{voigt2017eu}. The server wants to train a model on all local
datasets. To this end, it initializes the model and sends it to a random
subset of clients. Each client trains the model on its local dataset and
sends the trained model back to the server. The server accumulates all
trained models into an updated model for the next iteration and repeats the
process for several rounds until some termination criterion is met. This
procedure enables the server to train a model without accessing any local
datasets. 

Today's neural network models often have millions or even billions
\citep{gpt3} of parameters, which makes high communication costs a
concern in FL. In fact, \citet{carbon} suggest that communication between clients
and server may account for over 70\% of energy consumption in FL. Reducing
communication in FL is an attractive area of research because it lowers
bandwidth requirements, energy consumption and training time.

Communication in FL occurs in two phases: Sending parameters from the
server to clients (\emph{downlink}) and sending updated parameters from
clients to the server (\emph{uplink}). Uplink bandwidth usually imposes a
tighter bottleneck than downlink bandwidth. This has several reasons. For
one, the average global mobile upload bandwidth is currently less than one
fourth of the download bandwidth \citep{speedtest}. For another, FL downlink
communication sends the same parameters to each client. Broadcasting
parameters is usually more efficient than the accumulation of parameters
from different clients that is required for uplink communication \citep{LFL,
FedPAQ}. For these reasons, we  seek to compress uplink communication.

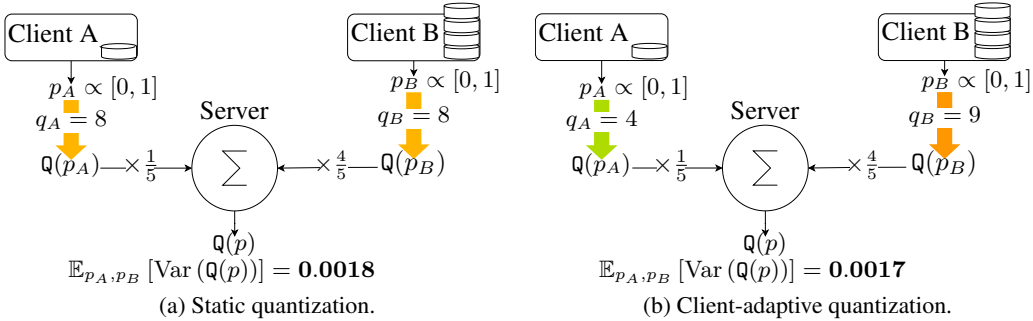
\begin{figure}[h] 
  \begin{subfigure}[b]{0.50\textwidth}
    \def\svgwidth{18em}
    \input{samplequant_complex_static.pdf_tex}
    \caption{Static quantization.}
    \label{fig:static_quantization}
  \end{subfigure}  
  \hfill
  \begin{subfigure}[b]{0.50\textwidth}
    \def\svgwidth{18em}
    \input{samplequant_complex_dynamic.pdf_tex}
    \label{fig:dynamic_quantization}
    \caption{Client-adaptive quantization.}
  \end{subfigure} 
  \caption{Static quantization vs. client-adaptive quantization when
  accumulating parameters $p_A$ and $p_B$. (a): Static quantization uses the
  same quantization level for $p_A$ and $p_B$. (b) Client-adaptive
  quantization uses a slightly higher quantization level for $p_B$ because
  $p_B$ is weighted more heavily. This allows us to use a significantly lower quantization level $q_A$ for $p_A$ while keeping the quantization error measure
  $\mathrm{E}_{p_A,p_B}\left[\mathrm{Var}\left(\quantizer(p)\right)\right]$
  roughly constant. Since communication is approximately proportional to
  $q_A + q_B$, client-adaptive quantization communicates less data.}
  \label{fig:clientdynamicquant}
  \vspace*{-2mm}
\end{figure}

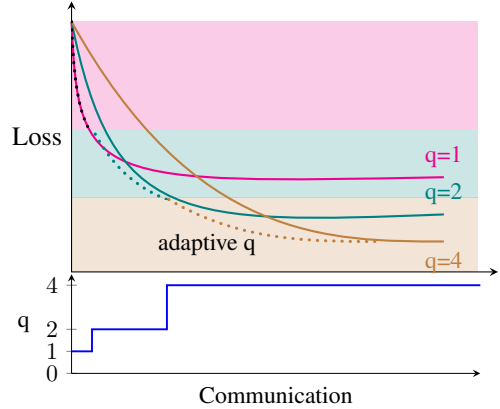
\begin{wrapfigure}{r}{0.49\textwidth}
  \vspace*{-5mm}
  \begin{tikzpicture}[scale=0.9]
    \fill [magenta, opacity=0.2] (0,2.1) rectangle (6, 3.7);
    \fill [teal, opacity=0.2] (0,1.1) rectangle (6, 2.1);
    \fill [brown, opacity=0.2] (0,0) rectangle (6, 1.1);
    \draw [<->, >=stealth] (0,4) -- (0,0) -- (6.3,0);
    \node [left] at (0,2) {Loss};
    \draw [magenta, thick](0,3.7) .. controls (0.1,1.3) and (0.3,1.3) .. (5.5, 1.4) node [above, magenta] {\small q=1};
    \draw [teal, thick,](0,3.7) .. controls (0.6,0.7) and (1.6,0.7) .. (5.5, 0.85) node [above, teal] {\small q=2};
    \draw [brown, thick](0,3.7) .. controls (1.8,0.45) and (3.8,0.45) .. (5.5, 0.45) node [below, brown] {\small q=4};
    \begin{scope}
      \clip (0, 2.1) rectangle (6, 4);
      \draw [black, very thick, dotted, line cap=round, dash pattern=on 0pt off 2.5\pgflinewidth](0,3.7) .. controls (0.1,1.3) and (0.3,1.3) .. (5.5, 1.4);
    \end{scope}
    \begin{scope}[shift={(-1.02,0.00)}]
      \clip (0, 0) rectangle (6, 1.1);
      \draw [brown, dotted, very thick, line cap=round, dash pattern=on 0pt off 2.5\pgflinewidth](0,4) .. controls (1.8,0.45) and (3.8,0.45) .. (5.5, 0.45);
    \end{scope}
    \begin{scope}[shift={(-0.245,-0.00)}]
      \clip (0, 1.065) rectangle (6, 2.1);
      \draw [teal, very thick, dotted, line cap=round, dash pattern=on 0pt off 2.5\pgflinewidth](0,4) .. controls (0.6,0.7) and (1.6,0.7) .. (5.5, 0.85);
    \end{scope}
    \node [below] at (2, 0.7) {\small adaptive q};

    \begin{scope}[shift={(0,-1.5)}]
      \begin{axis}[%
        ,xlabel=Communication
        ,xlabel near ticks
        ,ylabel near ticks
        ,ylabel=q
        ,axis x line = bottom,axis y line = left
        ,ytick={0, 1,2,4}
        ,xtick=\empty
        ,ymax=4.5 
        ,ymin=0
        ,width=3in
        ,height=1.2in
      ]
      \addplot+[const plot, no marks, thick] coordinates {(0,1) (0.15,1) (0.15,2) (0.7,2) (0.7,4) (3,4)};
      \end{axis}
    \end{scope}

  \end{tikzpicture} 
  \caption{Time-adaptive quantization. A small quantization level (q)
  decreases the loss with less communication than a large q, but converges
  to a higher loss. This motivates an adaptive quantization strategy that uses a small q as long as 
  it is beneficial and then switches over to a large q. We generalize this idea into an algorithm
  that monotonically increases q based on the training loss.}
  \label{fig:timedynamicquant}
\end{wrapfigure}

A large class of compression algorithms for FL apply some lossy quantizer $\quantizer$, optionally
followed by a lossless compression stage. $\quantizer$ usually provides a ``quantization level''
hyperparameter $q$ to control the coarseness of quantization (e.g. the number of bins
for fixed-point quantization). When $q$ is kept constant during training, we speak of
\emph{static quantization}. When $q$ changes, we speak of \emph{adaptive quantization}. Adaptive
quantization can exploit asymmetries in the FL framework to minimize communication. One such
asymmetry lies in FL's training time, where we observe that early training rounds can use a lower
$q$ without affecting convergence. \Cref{fig:timedynamicquant} illustrates how \emph{time-adaptive
quantization} leverages this phenomenon to minimize communication. Another asymmetry lies in FL's
client space, because most FL algorithms weight client contributions to the global model
proportional to their local dataset sizes. \Cref{fig:clientdynamicquant} illustrates how
\emph{client-adaptive quantization} can minimize the quantization error. Intuitively, FL clients
with greater weighting should have a greater communication budget and our proposed client-adaptive
quantization achieves this in a principled way. To this end, we introduce the expected variance of
an accumulation of quantized parameters, $\E[\Var(\sum\quantizer(p))]$, as a measure of the
quantization error. Our client-adaptive quantization algorithm then assigns clients minimal
quantization levels, subject to a fixed $\E[\Var(\sum\quantizer(p))]$. This lowers the amount of
data communicated from clients to the server, without increasing the quantization error.

DAdaQuant (Doubly Adaptive Quantization) combines time- and client-adaptive quantization with an adaptation
of the QSGD fixed-point quantization algorithm to achieve state-of-the-art
FL uplink compression.
In this paper, we make the following contributions:

\begin{itemize}[noitemsep,leftmargin=*] 
  \item We introduce the concept of client-adaptive quantization and develop algorithms for time-
  and client-adaptive quantization that are computationally efficient, empirically Pareto optimal
  and compatible with arbitrary FL quantizers. Our client-adaptive quantization is provably optimal
  for stochastic fixed-point quantizers. 
  
  \item We create Federated QSGD as an adaptation of the stochastic
  fixed-point quantizer QSGD that works with FL. Federated QSGD outperforms
  all other quantizers, establishing a strong baseline for FL compression with
  static quantization.

  \item We combine time- and client-adaptive quantization into DAdaQuant. We
  demonstrate DAdaQuant's state-of-the-art compression by empirically
  comparing it against several competitive FL compression
  algorithms.
\end{itemize}


\section{Related Work}

FL research has explored several approaches to reduce communication. We
identify three general directions.

First, there is a growing interest of
investigating FL algorithms that can converge in fewer rounds. FedAvg
\citep{FedAvg} achieves this with prolonged local training, while FOLB
\citep{folb} speeds up convergence through a more principled client
sampling. Since communication is proportional to the number of training
rounds, these algorithms effectively reduce communication.

Secondly, communication can be reduced by reducing the model
size because the model size is proportional to the amount of training communication.
PruneFL \citep{fedprune} progressively prunes the model over the course of
training, while AFD \citep{feddropout} only trains submodels on clients.

Thirdly, it is possible to directly compress FL training communication. FL
compression algorithms typically apply techniques like top-k sparsification
\citep{fedzip, fetchsgd} or quantization \citep{FedPAQ, uveqfed} to
parameter updates, optionally followed by lossless compression. Our work
applies to quantization-based compression algorithms. It is partially based
on QSGD \citep{QSGD}, which combines lossy fixed-point quantization with a
lossless compression algorithm to compress gradients communicated in
distributed training. DAdaQuant adapts QSGD into Federated QSGD, which works
with Federated Learning. DAdaQuant also draws inspiration from FedPAQ
\citep{FedPAQ}, the first FL framework to use lossy compression based on
model parameter update quantization. However, FedPAQ does not explore the
advantages of additional lossless compression or adaptive quantization.
UVeQFed \citep{uveqfed} is an FL compression algorithm that generalizes
scalar quantization to vector quantization and subsequently employs lossless
compression with arithmetic coding. Like FedPAQ, UVeQFed also limits itself
to a single static quantization level.

Faster convergence, model size reduction and communication compression are
orthogonal techniques, so they can be combined for further communication
savings. For this paper, we limit the scope of empirical comparisons to
quantization-based FL compression algorithms.

For quantization-based compression for model training, prior works have
demonstrated that DNNs can be successfully trained in low-precision
\citep{banner2018scalable,gupta2015deep,sun2019hybrid}. There are also
several adaptive quantization algorithms for training neural networks in a
non-distributed setting. \Citet{adaparams} use different quantization levels
for different parameters of a neural network. FracTrain \citep{FracTrain}
introduced multi-dimensional adaptive quantization by developing
time-adaptive quantization and combining it with parameter-adaptive
quantization. However, FracTrain uses the current loss to decide on the
quantization level. FL generally can only compute local client losses that
are too noisy to be practical for FracTrain. AdaQuantFL introduces
time-adaptive quantization to FL, but requires the global loss
\citep{adaquantfl}. To compute the global loss, AdaQuantFL has to
communicate with every client each round. We show in
\Cref{sec:experimentsresults} that this quickly becomes impractical as the
number of clients grows. DAdaQuant's time-adaptive quantization overcomes
this issue without compromising on the underlying FL communication. In
addition, to the best of our knowledge, DAdaQuant is the first algorithm to
use client-adaptive quantization.

\section{The DAdaQuant method}

\subsection{Federated Learning}

Federated Learning assumes a client-server topology with a set $\sC =
\{c_i|i \in \{1,2...N\}\}$ of $N$ clients that are connected to a single server. Each client
$c_k$ has a local dataset $D_k$ from the local data distribution
$\mathcal{D}_k$.
Given a model $M$ with parameters $\vp$, a loss function
$f_{\vp}(d\in D_k)$ and the local loss $F_k(\vp) =
\frac{1}{|D_k|}\sum_{d \in D_k} f_\vp(d)$, FL seeks to minimize the global
loss $\globloss(\vp) = \sum_{k=1}^{N} \frac{|D_k|}{\sum_l|D_l|}F_k(\vp)\label{eq:flobjective}$.

\subsection{Federated Averaging (FedAvg)}

DAdaQuant makes only minimal assumptions about the FL algorithm. Crucially,
DAdaquant can complement FedAvg \citep{FedAvg}, which is representative of a
large class of FL algorithms.

FedAvg trains the model $M$ over several rounds. In each round $t$, FedAvg
sends the model parameters $\vp_t$ to a random subset $\sS_t$ of $K$ clients
who then optimize their local objectives $F_k(\vp_t)$ and send the updated
model parameters $\vp_{t+1}^k$ back to the server. The server accumulates
all parameters into the new global model $\vp_{t+1} = \sum_{k\in\sS_t}
\frac{|D_k|}{\sum_j |D_j|} \vp_{t+1}^{k}$\; and starts the next round.
\Cref{alg:fedavgdadaquant} lists FedAvg in detail. For our experiments, we
use the FedProx \citep{FedProx} adaptation of FedAvg. FedProx improves the
convergence of FedAvg by adding the proximal term
$\frac{\mu}{2}\|\vp_{t+1}^k - \vp_t\|^2$ to the local objective
$F_k(\vp_{t+1}^k)$ in \Cref{alg:fedavgobjective} of
\Cref{alg:fedavgdadaquant}.

\subsection{Quantization with Federated QSGD}
\label{sec:qsgd}

While DAdaQuant can be applied to any quantizer with a configurable
quantization level, it is optimized for fixed-point quantization. We
introduce Federated QSGD as a competitive fixed-point quantizer on top of
which DAdaQuant is applied. 

In general, fixed-point quantization uses a quantizer $\quantizer_q$ with
quantization level $q$ that splits $\R_{\geq0}$ and $\R_{\leq0}$ into $q$ intervals each.
$\quantizer_q(p)$ then returns the sign of $p$ and $|p|$ rounded to one of
the endpoints of its encompassing interval. $\quantizer_q(\vp)$ quantizes the
vector $\vp$ elementwise.

We design DAdaQuant's quantization stage based on QSGD, an efficient
fixed-point quantizer for state-of-the-art gradient compression.
QSGD quantizes a vector $\vp$ in
three steps:
\begin{enumerate}[noitemsep]
  \item Quantize $\vp$ as $\quantizer_q(\frac{\vp}{||\vp||_2})$ into $q$ bins in $[0,1]$, storing signs and $||\vp||_2$ separately. (\emph{lossy})
  \item Encode the resulting integers with 0 run-length encoding. (\emph{lossless})
  \item Encode the resulting integers with Elias $\omega$ coding. (\emph{lossless})
\end{enumerate}

QSGD has been designed specifically for quantizing gradients. This makes it
not directly applicable to parameter compression. To overcome this
limitation, we apply difference coding to uplink compression, first
introduced to FL by FedPAQ. Each client $c_k$ applies $\quantizer_q$ to the
\emph{parameter updates} $\vp^k_{t+1}-\vp_t$ (cf.
\Cref{alg:fedavgclientreturn} of \Cref{alg:fedavgdadaquant}) and sends them to the server. The server
keeps track of the previous parameters $\vp_t$ and accumulates the quantized
parameter updates into the new parameters as $\vp_{t+1} = \vp_t + \sum_{k\in\sS_t}
\frac{|D_k|}{\sum_l |D_l|} \quantizer_q(\vp^k_{t+1}-\vp_t)$ (cf.
\Cref{alg:fedavgaccumulate} of \Cref{alg:fedavgdadaquant}). We find that QSGD works well with parameter
updates, which can be regarded as an accumulation of gradients over several
training steps. We call this adaptation of QSGD \emph{Federated QSGD}.

\subsection{Time-adaptive quantization}

Time-adaptive quantization uses a different quantization level $q_t$ for
each round $t$ of FL training. DAdaQuant chooses $q_t$ to minimize
communication costs without sacrificing accuracy. To this end, we find that
lower quantization levels suffice to initially reduce the loss, while partly
trained models require higher quantization levels to further improve (as illustrated in \Cref{fig:timedynamicquant}).
FracTrain is built on similar observations for non-distributed training.  
Therefore, we design DAdaQuant to mimic FracTrain in monotonically
increasing $q_t$ as a function of $t$ and using the training loss to inform
increases in $q_t$.

When $q$ is too low, FL converges prematurely. Like FracTrain, DAdaQuant
monitors the FL loss and increases $q$ when it converges. Unlike FracTrain,
there is no single centralized loss function to evaluate and unlike
AdaQuantFL, we do not assume availability of global training loss
$\globloss(\vp_t)$. Instead, we estimate $\globloss(\vp_t)$ as the average
local loss $\locloss_t = \sum_{k\in\sS_t} \frac{|D_k|}{\sum_l
|D_l|}F_k(\vp_t)$ where $\sS_t$ is the set of clients sampled at round $t$.
Since $\sS_t$ typically consists of only a small fraction of all clients,
$\locloss_t$ is a very noisy estimate of $\globloss(\vp_t)$. This makes it
unsuitable for convergence detection. Instead, DAdaQuant tracks a running
average loss $\raloss_{t} = \psi \raloss_{t-1} + (1-\psi) \locloss_{t}$.

We initialize $q_1 = q_\text{min}$ for some $q_\text{min} \in \sN$.
DAdaQuant determines training to converge whenever $\raloss_{t} \geq
\raloss_{t+1-\phi}$ for some $\phi \in \sN$ that specifies the number of
rounds across which we compare $\raloss$. On convergence, DAdaQuant sets
$q_{t} = 2q_{t-1}$ and keeps the quantization level fixed for at least
$\phi$ rounds to enable reductions in $\globloss$ to manifest in $\raloss$.
Eventually, the training loss converges regardless of the quantization
level. To avoid unconstrained quantization increases on convergence, we
limit the quantization level to $q_\text{max}$.

The following equation
summarizes DAdaQuant's time-adaptive quantization:
$$
q_{t} \longleftarrow
\begin{cases}
  q_{\text{min}} & t = 0 \\
  2q_{t-1} & t > 0 \text{ and } \raloss_{t-1} \geq \raloss_{t-\phi} \text{ and } t > \phi \text{ and } 2q_{t-1} < q_{\text{max}} \text{ and } q_{t-1} = q_{t-\phi} \\
  q_{t-1} & \text{else}
\end{cases}
$$

\begin{figure}[]\small\centering
\begin{subtable}[]{0.492\textwidth}
\begin{tabular}{l|l|lllll}
\multicolumn{2}{r|}{Round} & \multicolumn{1}{l|}{1}            & \multicolumn{1}{l|}{2}            & \multicolumn{1}{l|}{3}            & \multicolumn{1}{l|}{4}            & \multicolumn{1}{l|}{5}            \\ \hline
Client      & Samples      & \multicolumn{5}{c}{Quantization level}                                                                                                                                            \\ \hline
A           & 1            & \multicolumn{1}{l|}{}             & \multicolumn{1}{l|}{}             & \multicolumn{1}{l|}{}             & \multicolumn{1}{l|}{\gradient{8}} & \multicolumn{1}{l|}{}             \\
B           & 2            & \multicolumn{1}{l|}{\gradient{8}} & \multicolumn{1}{l|}{\gradient{8}} & \multicolumn{1}{l|}{\gradient{8}} & \multicolumn{1}{l|}{\gradient{8}} & \multicolumn{1}{l|}{}             \\
C           & 3            & \multicolumn{1}{l|}{\gradient{8}} & \multicolumn{1}{l|}{}             & \multicolumn{1}{l|}{\gradient{8}} & \multicolumn{1}{l|}{}             & \multicolumn{1}{l|}{\gradient{8}} \\
D           & 4            & \multicolumn{1}{l|}{}             & \multicolumn{1}{l|}{\gradient{8}} & \multicolumn{1}{l|}{}             & \multicolumn{1}{l|}{}             & \multicolumn{1}{l|}{\gradient{8}}
\end{tabular}
\caption{Static quantization.}
\end{subtable}\hfill
\begin{subtable}[h]{0.492\textwidth}
\begin{tabular}{l|l|lllll}
\multicolumn{2}{r|}{Round} & \multicolumn{1}{l|}{1}            & \multicolumn{1}{l|}{2}            & \multicolumn{1}{l|}{3}            & \multicolumn{1}{l|}{4}            & \multicolumn{1}{l|}{5}            \\ \hline
Client      & Samples      & \multicolumn{5}{c}{Quantization level}                                                                                                                                            \\ \hline
A           & 1            & \multicolumn{1}{l|}{}             & \multicolumn{1}{l|}{}             & \multicolumn{1}{l|}{}             & \multicolumn{1}{l|}{\gradient{4}} & \multicolumn{1}{l|}{}             \\
B           & 2            & \multicolumn{1}{l|}{\gradient{1}} & \multicolumn{1}{l|}{\gradient{2}} & \multicolumn{1}{l|}{\gradient{2}} & \multicolumn{1}{l|}{\gradient{4}} & \multicolumn{1}{l|}{}             \\
C           & 3            & \multicolumn{1}{l|}{\gradient{1}} & \multicolumn{1}{l|}{}             & \multicolumn{1}{l|}{\gradient{2}} & \multicolumn{1}{l|}{}             & \multicolumn{1}{l|}{\gradient{8}} \\
D           & 4            & \multicolumn{1}{l|}{}             & \multicolumn{1}{l|}{\gradient{2}} & \multicolumn{1}{l|}{}             & \multicolumn{1}{l|}{}             & \multicolumn{1}{l|}{\gradient{8}}
\end{tabular}
\caption{Time-adaptive quantization.}
\end{subtable}

  \begin{subtable}[h]{0.492\textwidth}
  \begin{tabular}{l|l|lllll}
  \multicolumn{2}{r|}{Round} & \multicolumn{1}{l|}{1}            & \multicolumn{1}{l|}{2}            & \multicolumn{1}{l|}{3}            & \multicolumn{1}{l|}{4}            & \multicolumn{1}{l|}{5}            \\ \hline
  Client      & Samples      & \multicolumn{5}{c}{Quantization level}                                                                                                                                            \\ \hline
  A           & 1            & \multicolumn{1}{l|}{}             & \multicolumn{1}{l|}{}             & \multicolumn{1}{l|}{}             & \multicolumn{1}{l|}{\gradient{6}} & \multicolumn{1}{l|}{}             \\
  B           & 2            & \multicolumn{1}{l|}{\gradient{7}} & \multicolumn{1}{l|}{\gradient{6}} & \multicolumn{1}{l|}{\gradient{7}} & \multicolumn{1}{l|}{\gradient{9}} & \multicolumn{1}{l|}{}             \\
  C           & 3            & \multicolumn{1}{l|}{\gradient{9}} & \multicolumn{1}{l|}{}             & \multicolumn{1}{l|}{\gradient{9}} & \multicolumn{1}{l|}{}             & \multicolumn{1}{l|}{\gradient{7}} \\
  D           & 4            & \multicolumn{1}{l|}{}             & \multicolumn{1}{l|}{\gradient{9}} & \multicolumn{1}{l|}{}             & \multicolumn{1}{l|}{}             & \multicolumn{1}{l|}{\gradient{9}}
  \end{tabular}\hspace{0.47em}
  \caption{Client-adaptive quantization.}
  \end{subtable}
  \begin{subtable}[h]{0.492\textwidth}
  \begin{tabular}{l|l|lllll}
  \multicolumn{2}{r|}{Round} & \multicolumn{1}{l|}{1}            & \multicolumn{1}{l|}{2}            & \multicolumn{1}{l|}{3}            & \multicolumn{1}{l|}{4}            & \multicolumn{1}{l|}{5}            \\ \hline
  Client      & Samples      & \multicolumn{5}{c}{Quantization level}                                                                                                                                            \\ \hline
  A           & 1            & \multicolumn{1}{l|}{}             & \multicolumn{1}{l|}{}             & \multicolumn{1}{l|}{}             & \multicolumn{1}{l|}{\gradient{3}} & \multicolumn{1}{l|}{}             \\
  B           & 2            & \multicolumn{1}{l|}{\gradient{1}} & \multicolumn{1}{l|}{\gradient{1}} & \multicolumn{1}{l|}{\gradient{2}} & \multicolumn{1}{l|}{\gradient{5}} & \multicolumn{1}{l|}{}             \\
  C           & 3            & \multicolumn{1}{l|}{\gradient{1}} & \multicolumn{1}{l|}{}             & \multicolumn{1}{l|}{\gradient{2}} & \multicolumn{1}{l|}{}             & \multicolumn{1}{l|}{\gradient{7}} \\
  D           & 4            & \multicolumn{1}{l|}{}             & \multicolumn{1}{l|}{\gradient{1}} & \multicolumn{1}{l|}{}             & \multicolumn{1}{l|}{}             & \multicolumn{1}{l|}{\gradient{9}}
  \end{tabular}
  \caption{Time-adaptive and client-adaptive quantization.}
  \end{subtable}
  \caption{Exemplary quantization level assignment for 4 FL clients that train over 5 rounds. Each round, two clients get sampled for training.}
  \label{fig:quantexample}
\end{figure}

\subsection{Client-adaptive quantization}

FL algorithms typically accumulate each parameter $p_i$ over all clients
into a weighted average $p = \sum_{i=1}^K{w_ip_i}$ (see \Cref{alg:fedavgdadaquant}).
Quantized FL communicates and accumulates quantized parameters
$\quantizer_q(p) = \sum_{i=1}^K{w_i\quantizer_q(p_i)}$ where $q$ is the
quantization level. We define the quantization error $e^q_p$ as $e^q_p = |p
- \quantizer_q(p)|$. We observe that $\E_{p_1\ldots
p_K}[\Var(\quantizer_q(p))]$ is a useful statistic of the quantization error
because it strongly correlates with the loss added by quantization. For
a stochastic, unbiased fixed-point compressor like Federated QSGD, $\E_{p_1\ldots
p_K}[\Var(\quantizer_q(p))]$ equals $\E_{p_1\ldots p_K}[\Var(e^q_p)]$ and
can be evaluated analytically.

We observe in our experiments that communication cost per client is roughly
a linear function of Federated QSGD's quantization level $q$. This means that the
communication cost per round is proportional to $Q = Kq$. We call $Q$ the
communication budget and use it as a proxy measure of communication cost.

Client-adaptive quantization dynamically adjusts the quantization level of each client. This means
that even within a single round, each client $c_k$ can be assigned a different quantization level $q_k$. The
communication budget of client-adaptive quantization is then $Q = \sum_{k=1}^K{q_k}$ and
$\quantizer_q(p)$ generalizes to $\quantizer_{q_1\ldots q_K}(p) = \sum_{i=1}^K{w_i\quantizer_{q_i}(p_i)}$. We devise an algorithm that chooses $q_k$ to minimize $Q$ subject to $\E_{p_1\ldots
p_K}[\Var(e^{q_1\ldots q_K}_p)] = \E_{p_1\ldots p_K}[\Var(e^q_p)]$ for a given $q$. Thus, our algorithm effectively minimizes
communication costs while maintaining a quantization error similar to static quantization. \Cref{theorem:q} provides us with an analytical formula for quantization
levels $q_1\ldots q_K$. 

\begin{theorem}{Given parameters $p_1\ldots p_k\sim\mathcal{U}[-t, t]$ and quantization level $q$, $\min_{q_1\ldots q_K}\sum_{i=1}^K{q_i}$
  subject to $\E_{p_1\ldots p_K}[\Var(e^{q_1\ldots q_K}_p)] = \E_{p_1\ldots p_K}[\Var(e^q_p)]$ is minimized by $q_i = \sqrt{\frac{a}{b}}\times w_i^{2/3}$ where $a = {\sum_{j=1}^K w_j^{2/3}}$ and $b = {\sum_{j=1}^K \frac{w_j^2}{q^2}}$.}
\label{theorem:q}
\end{theorem}

DAdaQuant applies \Cref{theorem:q} to lower communication costs while
maintaining the same loss as static quantization does with a fixed $q$. To
ensure that quantization levels are natural numbers, DAdaQuant approximates
the optimal real-valued solution as $q_i = \max(1,
\text{round}(\sqrt{\frac{a}{b}}\times w_i^{2/3}))$. \Cref{sec:proofs} gives a detailed proof
of \Cref{theorem:q}. To the best of our knowledge, DAdaQuant is the first algorithm
to use client-adaptive quantization.

\begin{algorithm}[H]
  \SetAlgoLined
  \Fn(){\FServer{}} {
      Initialize $w_i = \frac{|D_i|}{\sum_j |D_j|}$ for all $i\in[1,\ldots,N]$\;
      \For(){$t = 0,\dots,T-1$}{
      Choose $\sS_t \subset \sC$ with $|\sS_t| = K$, including each $c_k \in \sC$ with uniform probability\;
      \colorbox{brown!40}{
        $q_{t} \longleftarrow
        \begin{cases}
          q_{\text{min}} & t = 0 \\
          2q_{t-1} & t > 0 \text{ and } \raloss_{t-1} \geq \raloss_{t-\phi} \text{ and } t > \phi \text{ and } q_{t} \leq q_{\text{max}} \text{ and } q_{t-1} = q_{t-\phi} \\
          q_{t-1} & \text{else}
        \end{cases}$}\;
      \For(in parallel){$c_k \in \sS_t$}{
          \colorbox{magenta!40}{
          $q_t^k \longleftarrow
          \sqrt{\sum_{j=1}^K w_j^{2/3}/\sum_{j=1}^K \frac{w_j^2}{q^2}}
          $
          }\;
          $Send(c_k, \vp_t, ${$q_t^k$}$)$\;    
          $Receive(c_k, \vp_{t+1}^k,${$\locloss_t^k$}$)$\;
      }
      $\vp_{t+1} \longleftarrow \sum_{k\in\sS_t} w_k \vp_{t+1}^{k}$\;
      \label{alg:fedavgaccumulate}
      \colorbox{brown!40}{
        $\locloss_{t} \longleftarrow \sum_{k\in\sS_t} w_k \locloss_t^k$
      }\;
      \colorbox{brown!40}{
        $\raloss_{t} \longleftarrow 
        \begin{cases}
         \locloss_0 & t = 0 \\
        \psi \raloss_{t-1} + (1-\psi) \locloss_{t} & \textrm{else} \\
        \end{cases}$
        }\;
      } 
  }
  \Fn(){\FClient{$c_k$}} {
      $Receive(\textrm{Server}, \vp_t,$\,{$q_t^k$}$)$\;
      \colorbox{brown!40}{$\locloss_t^k \longleftarrow F_k(\vp_{t})$}\;
      $\vp_{t+1}^k \longleftarrow$ $F_k(\vp_{t+1}^k)$ trained with  SGD for $E$ epochs with learning rate $\eta$\;
      \label{alg:fedavgobjective}
      $Send(\textrm{Server},\,$\colorbox{teal!40}{$\quantizer_{q_t^k}(\vp_{t+1}^k)$}$, ${$\locloss_t^k$}$)$\;
      \label{alg:fedavgclientreturn}
  }
  \caption{The FedAvg and DAdaQuant algorithms. The uncolored lines list FedAvg. Adding the colored lines creates DAdaQuant. \textrect{teal} --- quantization, \textrect{magenta} --- client-adaptive quantization, \textrect{brown} --- time-adaptive quantization.}
  \label{alg:fedavgdadaquant}
\end{algorithm}

\subsection{Doubly-adaptive quantization (DAdaQuant)}

DAdaQuant combines the time-adaptive and client-adaptive quantization
algorithms described in the previous sections. At each round $t$,
time-adaptive quantization determines a preliminary quantization level
$q_t$. Client-adaptive quantization then finds the client quantization
levels $q_t^k, k \in \{1, \ldots, K\}$ that minimize $\sum_{i=1}^K{q_i}$
subject to $\E_{p_1\ldots p_K}[\Var(e^{q_1\ldots q_K}_p)] = \E_{p_1\ldots
p_K}[\Var(e^q_p)]$. \Cref{alg:fedavgdadaquant} lists DAdaQuant in detail.
\Cref{fig:quantexample} gives an example of how our time-adaptive,
client-adaptive and doubly-adaptive quantization algorithms set quantization levels.

\Citet{FedPAQ} prove the convergence of FL with quantization for convex
and non-convex cases as long as the quantizer $\quantizer$ is (1) unbiased
and (2) has a bounded variance. These convergence results extend to
DAdaQuant when combined with any quantizer that satisfies (1) and (2) for
DAdaQuant's minimum quantization level $q=1$. Crucially, this includes
Federated QSGD. 

We highlight DAdaQuant's low overhead and general applicability. The
computational overhead is dominated by an additional evaluation epoch per
round per client to compute $\raloss_t$, which is negligible when training
for many epochs per round. DAdaQuant can compliment any FL algorithm that
trains models over several rounds and accumulates a weighted average of
client parameters. Most FL algorithms, including FedAvg, follow this design.

\section{Experiments}
\label{sec:experiments}

\subsection{Experimental details}

\textbf{Evaluation} We use DAdaQuant with Federated QSGD to train
different models with FedProx on different datasets for a fixed number of rounds.
We monitor the test loss and accuracy at fixed intervals and measure
uplink communication at every round across all devices.

\textbf{Models \& datasets} We select a broad and diverse set of five
models and datasets to demonstrate the general applicability of DAdaQuant.
To this end, we use DAdaQuant to train a linear model, CNNs and LSTMs of
varying complexity on a federated synthetic dataset (\dataset{Synthetic}),
as well as two federated image datasets (\dataset{FEMNIST} and
\dataset{CelebA}) and two federated natural language datasets
(\dataset{Sent140} and \dataset{Shakespeare}) from the LEAF \citep{LEAF}
project for standardized FL research. We refer to
\Cref{sec:models_datasets_detailed} for more information on the
models, datasets, training objectives and implementation.

\textbf{System heterogeneity}
In practice, FL has to cope with clients that have different compute
capabilities. We follow \citet{FedProx} and simulate this \emph{system
heterogeneity} by randomly reducing the number of epochs to $E'$ for a
random subset $\sS_t' \subset \sS_t$ of clients at each round $t$, where
$E'$ is sampled from $[1, \ldots, E]$ and $|\sS_t'| = 0.9K$. 

\textbf{Baselines}
We compare DAdaQuant against competing quantization-based algorithms for FL
parameter compression, namely Federated QSGD, FedPAQ \citep{FedPAQ}, GZip
with fixed-point quantization (FxPQ + GZip), UVeQFed \citep{uveqfed} and
FP8. Federated QSGD (see \cref{sec:qsgd}) is our most important baseline
because it outperforms the other algorithms. FedPAQ only applies fixed-point
quantization, which is equivalent to Federated QSGD without lossless
compression. Similarly, FxPQ + GZip is equivalent to Federated QSGD with
Gzip for its lossless compression stages. UVeQFed generalizes scalar
quantization to vector quantization, followed by arithmetic coding. We apply
UVeQFed with the optimal hyperparameters reported by its authors. FP8
\citep{FP8} is a floating-point quantizer that uses an 8-bit floating-point
format designed for storing neural network gradients. We also evaluate all
experiments without compression to establish an accuracy benchmark.

\textbf{Hyperparameters} With the exception of \dataset{CelebA}, all our datasets
and models are also used by \citeauthor{FedProx}. We therefore adopt most of
the hyperparameters from \citeauthor{FedProx} and use LEAF's hyperparameters for \dataset{CelebA} \cite{LEAF}.
For all experiments, we sample 10 clients each round. We train
\dataset{Synthetic}, \dataset{FEMNIST} and \dataset{CelebA} for 500 rounds
each. We train \dataset{Sent140} for 1000 rounds due to slow convergence and
\dataset{Shakespeare} for 50 rounds due to rapid convergence. We use batch
size 10, learning rates 0.01, 0.003, 0.3, 0.8, 0.1 and $\mu$s (FedProx's
proximal term coefficient) 1, 1, 1, 0.001, 0 for \dataset{Synthetic},
\dataset{FEMNIST}, \dataset{Sent140}, \dataset{Shakespeare},
\dataset{CelebA} respectively. We randomly split the local datasets into
80\% training set and 20\% test set.

To select the quantization level $q$ for static quantization with Federated
QSGD, FedPAQ and FxPQ + GZip, we run a gridsearch over $q = 1, 2, 4, 8,
\ldots$ and choose for each dataset the lowest $q$ for which Federated QSGD
exceeds uncompressed training in accuracy. We set UVeQFed's ``coding rate''
hyperparameter $R=4$, which is the lowest value for which UVeQFed achieves
negligible accuracy differences compared to uncompressed training.
We set
the remaining hyperparameters of UVeQFed to the optimal values reported by
its authors. \Cref{sec:uveqfed} shows further experiments that compare against
UVeQFed with $R$ chosen to maximize its compression factor.

For DAdaQuant's time-adaptive quantization, we set $\psi$ to 0.9, $\phi$ to
${1/10}^{th}$ of the number of rounds and $q_\textrm{max}$ to the
quantization level $q$ for each experiment. For \dataset{Synthetic} and
\dataset{FEMNIST}, we set $q_\textrm{min}$ to 1. We find that
\dataset{Sent140}, \dataset{Shakespeare} and \dataset{CelebA} require a high
quantization level to achieve top accuracies and/or converge in few rounds.
This prevents time-adaptive quantization from increasing the quantization
level quickly enough, resulting in prolonged low-precision training that
hurts model performance. To counter this effect, we set $q_\textrm{min}$ to
$q_\textrm{max}/2$. This effectively results in binary time-adaptive
quantization with an initial low-precision phase with $q =
q_\textrm{max}/2$, followed by a high-precision phase with $q =
q_\textrm{max}$.

\subsection{Results}
\label{sec:experimentsresults}

\begin{wrapfigure}[18]{r}{0.40\textwidth}
  \vspace*{-1.7cm}  
  \begin{center}
    \scalebox{0.75}{
      \input{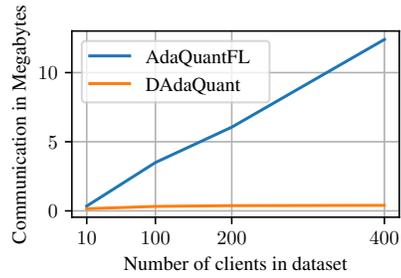}}
  \end{center}
  \caption{Comparison of AdaQuantFL and DAdaQuant. We plot the total
  client$\rightarrow$server communication required to train an MLR model on
  synthetic datasets with 10, 100, 200 and 400 clients. AdaQuantFL's
  communication increases linearly with the number of clients because it
  trains the model on all clients at each round. In contrast, DAdaQuant's
  communication does not change with the number of clients.} 
  \label{fig:adaquantfl} 
\end{wrapfigure}

We repeat the main experiments three times and report average results and
their standard deviation (where applicable). \Cref{tab:results} shows the
highest accuracy and total communication for each experiment.
\Cref{fig:pareto} plots the maximum accuracy achieved for any given amount
of communication.

\begin{table}[]
  \scriptsize
  \centering
  \begin{tabular}{ll@{\hskip -0mm}ll@{\hskip 2mm}ll@{\hskip 2mm}l}
    \multicolumn{1}{l|}{} & \multicolumn{2}{c|}{\textbf{Synthetic}} & \multicolumn{2}{c|}{\textbf{FEMNIST}} & \multicolumn{2}{c}{\textbf{Sent140}} \\ \hline
    \rowcolor[HTML]{EFEFEF} 
    \multicolumn{1}{l|}{\cellcolor[HTML]{EFEFEF}\textbf{Uncompressed}} & {$78.3\pm 0.3$} & \multicolumn{1}{l|}{\cellcolor[HTML]{EFEFEF}{$12.2$\,MB}} & {$77.7\pm 0.4$} & \multicolumn{1}{l|}{\cellcolor[HTML]{EFEFEF}{$\!\!\!\!\!132.1$\,GB}} & {$69.7\pm 0.5$} & {$43.9$\,GB} \\
    \multicolumn{1}{l|}{\textbf{Federated QSGD}} & {$-0.1\pm 0.1$} & \multicolumn{1}{l|}{{$17\times$}} & {$+0.7\pm 0.5$} & \multicolumn{1}{l|}{{$\!\!\!\!\!2809\times$}} & {$-0.0\pm 0.5$} & {$90\times$} \\
    \rowcolor[HTML]{EFEFEF} 
    \multicolumn{1}{l|}{\cellcolor[HTML]{EFEFEF}\textbf{FP8}} & {$\bm{+0.1\pm 0.4}$} & \multicolumn{1}{l|}{\cellcolor[HTML]{EFEFEF}{$4.0\times$ ($0.23\!\times\!\!\!\!\!\times$)}} & {$-0.1\pm 0.4$} & \multicolumn{1}{l|}{\cellcolor[HTML]{EFEFEF}{$\!\!\!\!\!4.0\times$ ($0.00\!\times\!\!\!\!\!\times$)}} & {$-0.2\pm 0.5$} & {$4.0\times$ ($0.04\!\times\!\!\!\!\!\times$)} \\
    \multicolumn{1}{l|}{\textbf{FedPAQ (FxPQ)}} & {$-0.1\pm 0.1$} & \multicolumn{1}{l|}{{$6.4\times$ ($0.37\!\times\!\!\!\!\!\times$)}} & {$+0.7\pm 0.5$} & \multicolumn{1}{l|}{{$\!\!\!\!\!11\times$ ($0.00\!\times\!\!\!\!\!\times$)}} & {$-0.0\pm 0.5$} & {$4.0\times$ ($0.04\!\times\!\!\!\!\!\times$)} \\
    \rowcolor[HTML]{EFEFEF} 
    \multicolumn{1}{l|}{\cellcolor[HTML]{EFEFEF}{\color[HTML]{333333} \textbf{FxPQ + GZip}}} & {\color[HTML]{333333} {$-0.1\pm 0.1$}} & \multicolumn{1}{l|}{\cellcolor[HTML]{EFEFEF}{\color[HTML]{333333} {$14\times$ ($0.82\!\times\!\!\!\!\!\times$)}}} & {\color[HTML]{333333} {$+0.6\pm 0.2$}} & \multicolumn{1}{l|}{\cellcolor[HTML]{EFEFEF}{\color[HTML]{333333} {$\!\!\!\!\!1557\times$ ($0.55\!\times\!\!\!\!\!\times$)}}} & {\color[HTML]{333333} {$-0.0\pm 0.6$}} & {\color[HTML]{333333} {$71\times$ ($0.79\!\times\!\!\!\!\!\times$)}} \\
    \multicolumn{1}{l|}{\textbf{UVeQFed}} & {$-0.5\pm 0.2$} & \multicolumn{1}{l|}{{$0.6\times$ ($0.03\!\times\!\!\!\!\!\times$)}} & {$-2.8\pm 0.5$} & \multicolumn{1}{l|}{{$\!\!\!\!\!12\times$ ($0.00\!\times\!\!\!\!\!\times$)}} & {$+0.0\pm 0.2$} & {$15\times$ ($0.16\!\times\!\!\!\!\!\times$)} \\ \hline
    \rowcolor[HTML]{EFEFEF} 
    \multicolumn{1}{l|}{\cellcolor[HTML]{EFEFEF}\textbf{DAdaQuant}} & {$-0.2\pm 0.4$} & \multicolumn{1}{l|}{\cellcolor[HTML]{EFEFEF}{$\bm{48\times}$ ($\bm{2.81\!\times\!\!\!\!\!\times}$)}} & {$+0.7\pm 0.1$} & \multicolumn{1}{l|}{\cellcolor[HTML]{EFEFEF}{$\!\!\!\!\!\bm{4772\times}$ ($\bm{1.70\!\times\!\!\!\!\!\times$})}} & {$-0.1\pm 0.4$} & {$\bm{108\times}$ ($\bm{1.19\!\times\!\!\!\!\!\times}$)} \\
    \multicolumn{1}{l|}{\textbf{DAdaQuant$_{\text{time}}$}} & {$-0.1\pm 0.5$} & \multicolumn{1}{l|}{{$37\times$ ($2.16\!\times\!\!\!\!\!\times$)}} & {$\bm{+0.8\pm 0.2}$} & \multicolumn{1}{l|}{{$\!\!\!\!\!4518\times$ ($1.61\!\times\!\!\!\!\!\times$)}} & {$-0.1\pm 0.6$} & {$93\times$ ($1.03\!\times\!\!\!\!\!\times$)} \\
    \rowcolor[HTML]{EFEFEF} 
    \multicolumn{1}{l|}{\cellcolor[HTML]{EFEFEF}\textbf{DAdaQuant$_{\text{clients}}$}} & {$+0.0\pm 0.3$} & \multicolumn{1}{l|}{\cellcolor[HTML]{EFEFEF}{$26\times$ ($1.51\!\times\!\!\!\!\!\times$)}} & {$+0.7\pm 0.4$} & \multicolumn{1}{l|}{\cellcolor[HTML]{EFEFEF}{$\!\!\!\!\!3017\times$ ($1.07\!\times\!\!\!\!\!\times$)}} & {$\bm{+0.1\pm 0.6}$} & {$105\times$ ($1.16\!\times\!\!\!\!\!\times$)} \\
     &  &  &  &  &  &  \\
    \multicolumn{1}{l|}{} & \multicolumn{2}{c|}{\textbf{Shakespeare}} & \multicolumn{2}{c}{\textbf{Celeba}} &  &  \\ \hhline{-----}
  
    \multicolumn{1}{l|}{\cellcolor[HTML]{EFEFEF}\textbf{Uncompressed}} & \cellcolor[HTML]{EFEFEF}{$\bm{49.9\pm 0.3}$} & \multicolumn{1}{l|}{\cellcolor[HTML]{EFEFEF}{$267.0$\,MB}} & \cellcolor[HTML]{EFEFEF}{$90.4\pm 0.0$} & \cellcolor[HTML]{EFEFEF}{$12.6$\,GB} &  &  \\
    \multicolumn{1}{l|}{\textbf{Federated QSGD}} & {$-0.5\pm 0.6$} & \multicolumn{1}{l|}{{$9.5\times$}} & {$-0.1\pm 0.1$} & {$648\times$} &  &  \\
    \multicolumn{1}{l|}{\cellcolor[HTML]{EFEFEF}\textbf{FP8}} & \cellcolor[HTML]{EFEFEF}{$-0.2\pm 0.4$} & \multicolumn{1}{l|}{\cellcolor[HTML]{EFEFEF}{$4.0\times$ ($0.42\!\times\!\!\!\!\!\times$)}} & \cellcolor[HTML]{EFEFEF}{$\bm{+0.0\pm 0.1}$} & \cellcolor[HTML]{EFEFEF}{$4.0\times$ ($0.01\!\times\!\!\!\!\!\times$)} &  &  \\
    \multicolumn{1}{l|}{\textbf{FedPAQ (FxPQ)}} & {$-0.5\pm 0.6$} & \multicolumn{1}{l|}{{$3.2\times$ ($0.34\!\times\!\!\!\!\!\times$)}} & {$-0.1\pm 0.1$} & {$6.4\times$ ($0.01\!\times\!\!\!\!\!\times$)} &  &  \\
    \multicolumn{1}{l|}{\cellcolor[HTML]{EFEFEF}\textbf{FxPQ + GZip}} & \cellcolor[HTML]{EFEFEF}{$-0.5\pm 0.6$} & \multicolumn{1}{l|}{\cellcolor[HTML]{EFEFEF}{$9.3\times$ ($0.97\!\times\!\!\!\!\!\times$)}} & \cellcolor[HTML]{EFEFEF}{$-0.1\pm 0.2$} & \cellcolor[HTML]{EFEFEF}{$494\times$ ($0.76\!\times\!\!\!\!\!\times$)} &  &  \\
    \multicolumn{1}{l|}{\textbf{UVeQFed}} & {$-0.0\pm 0.4$} & \multicolumn{1}{l|}{{$7.9\times$ ($0.83\!\times\!\!\!\!\!\times$)}} & {$-0.4\pm 0.3$} & {$31\times$ ($0.05\!\times\!\!\!\!\!\times$)} &  &  \\ \hhline{-----}
    \multicolumn{1}{l|}{\cellcolor[HTML]{EFEFEF}\textbf{DAdaQuant}} & \cellcolor[HTML]{EFEFEF}{$-0.6\pm 0.5$} & \multicolumn{1}{l|}{\cellcolor[HTML]{EFEFEF}{$\bm{21\times}$ ($\bm{2.21\!\times\!\!\!\!\!\times}$)}} & \cellcolor[HTML]{EFEFEF}{$-0.1\pm 0.1$} & \cellcolor[HTML]{EFEFEF}{$\bm{775\times}$ ($\bm{1.20\!\times\!\!\!\!\!\times}$)} &  &  \\
    \multicolumn{1}{l|}{\textbf{DAdaQuant$_{\text{time}}$}} & {$-0.5\pm 0.5$} & \multicolumn{1}{l|}{{$12\times$ ($1.29\!\times\!\!\!\!\!\times$)}} & {$-0.1\pm 0.2$} & {$716\times$ ($1.10\!\times\!\!\!\!\!\times$)} &  &  \\
    \multicolumn{1}{l|}{\cellcolor[HTML]{EFEFEF}\textbf{DAdaQuant$_{\text{clients}}$}} & \cellcolor[HTML]{EFEFEF}{$-0.4\pm 0.5$} & \multicolumn{1}{l|}{\cellcolor[HTML]{EFEFEF}{$16\times$ ($1.67\!\times\!\!\!\!\!\times$)}} & \cellcolor[HTML]{EFEFEF}{$-0.1\pm 0.0$} & \cellcolor[HTML]{EFEFEF}{$700\times$ ($1.08\!\times\!\!\!\!\!\times$)} &  & 
    \end{tabular}
  \caption{Top-1 test accuracies and total client$\rightarrow$server communication of all baselines, DAdaQuant,
  DAdaQuant$_\textrm{time}$ and DAdaQuant$_\textrm{clients}$. Entry $x \pm y\,\,\,\, p\!\times (q\!\times\!\!\!\!\times)$ denotes an accuracy difference of x\% w.r.t. the uncompressed
  accuracy with a standard deviation of y\%, a compression factor of $p$ w.r.t. the uncompressed
  communication and a compression factor of $q$ w.r.t. Federated QSGD.
  }
  \label{tab:results}
\end{table}

\paragraph{Baselines}

\Cref{tab:results} shows that the accuracy of most experiments lies within
the margin of error of the uncompressed experiments. This reiterates the
viability of quantization-based compression algorithms for communication
reduction in FL.
For all experiments, Federated QSGD achieves a significantly higher
compression factor than the other baselines. The authors of FedPAQ and
UVeQFed also compare their methods against QSGD and report them as superior.
However, FedPAQ is compared against ``unfederated'' QSGD that communicates
gradients after each local training step and UVeQFed is  
compared against QSGD without its lossless compression stages.

\paragraph{Time-adaptive quantization} The purely time-adaptive version of
DAdaQuant, DAdaQuant$_\textrm{time}$, universally outperforms Federated QSGD
and the other baselines in \Cref{tab:results}, achieving comparable accuracies
while lowering communication costs. DAdaQuant$_\textrm{time}$
performs particularly well on \dataset{Synthetic} and \dataset{FEMNIST},
where it starts from the lowest possible quantization level $q=1$. However,
binary time-adaptive quantization still measurably improves  over QSGD for
\dataset{Sent140}, \dataset{Shakespeare} and \dataset{Celeba}.

\Cref{fig:adaquantfl} provides empirical evidence that AdaQuantFL's communication scales linearly
with the number of clients. As a result, AdaQuantFL is prohibitively expensive for datasets with
thousands of clients such as \dataset{Celeba} and \dataset{Sent140}. DAdaQuant does not face this
problem because its communication is unaffected by the number of clients. 

\begin{figure}
  \vspace*{-0.7cm}
  \begin{center} 
    \hspace*{-0.5cm}
    \scalebox{0.75}{
      \input{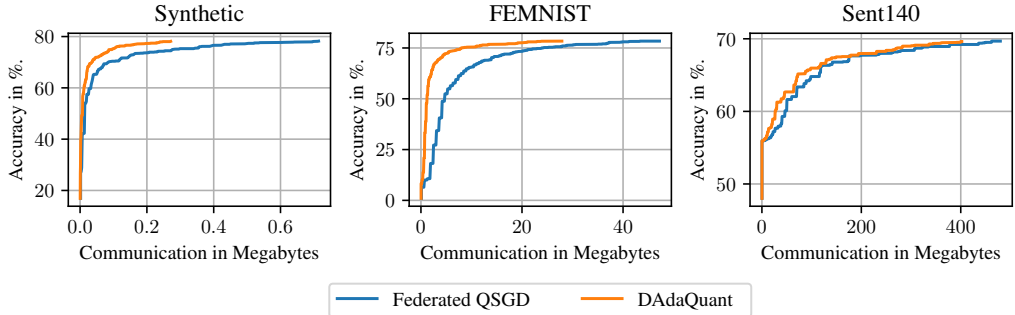}}
  \end{center}
  \vspace*{-0.3cm} 
  \caption{Communication-accuracy trade-off curves for training on
  \dataset{FEMNIST} with Federated QSGD and DAdaQuant. We plot the average highest
  accuracies achieved up to any given amount of client$\rightarrow$server communication.
  \Cref{sec:paretofull} shows curves for all datasets, with similar results.}
  \label{fig:pareto}
\end{figure}

\paragraph{Client-adaptive quantization}
The purely time-adaptive version of DAdaQuant, DAdaQuant$_\textrm{clients}$,
also universally outperforms Federated QSGD and the other baselines in
\Cref{tab:results}, achieving similar accuracies while lowering
communication costs. Unsurprisingly, the performance of
DAdaQuant$_\textrm{clients}$ is correlated with the coefficient of variation
$c_v = \frac{\sigma}{\mu}$ of the numbers of samples in the local datasets
with mean $\mu$ and standard deviation $\sigma$: \dataset{Synthetic}
($c_v=3.3$) and \dataset{Shakespeare} ($c_v=1.7$) achieve significantly
higher compression factors than \dataset{Sent140} ($c_v=0.3$),
\dataset{FEMNIST} ($c_v=0.4$) and \dataset{Celeba} ($c_v=0.3$).

\paragraph{DAdaQuant}
DAdaQuant outperforms DAdaQuant$_\textrm{time}$ and
DAdaQuant$_\textrm{clients}$ in communication while achieving similar
accuracies. The compression factors of DAdaQuant are roughly multiplicative
in those of DAdaQuant$_\textrm{clients}$ and DAdaQuant$_\textrm{time}$. This
demonstrates that we can effectively combine time- and client-adaptive
quantization for maximal communication savings.


\paragraph{Pareto optimality} \Cref{fig:pareto} shows that DAdaQuant
achieves a higher accuracy than the strongest baseline, Federated QSGD, for
any fixed accuracy. This means that DAdaQuant is Pareto optimal for the
datasets we have explored.
  
\section{Conclusion}

We introduced DAdaQuant as a computationally efficient and robust algorithm
to boost the performance of quantization-based FL compression algorithms. We
showed intuitively and mathematically how DAdaQuant's dynamic adjustment of
the quantization level across time and clients minimize
client$\rightarrow$server communication while maintaining convergence speed.
Our experiments establish DAdaQuant as nearly universally superior over
static quantizers, achieving state-of-the-art compression factors when
applied to Federated QSGD. The communication savings of DAdaQuant
effectively lower FL bandwidth usage, energy consumption and training time.
Future work may apply and adapt DAdaQuant to new quantizers, further pushing
the state of the art in FL uplink compression.
  
\section{Reproducibility Statement}

Our submission includes a repository with the source code for DAdaQuant and
for the experiments presented in this paper. All the datasets used in our
experiments are publicly available. Any post-processing steps of the
datasets are described in \Cref{sec:models_datasets_detailed}. To facilitate
the reproduction of our results, we have bundled all our source code,
dependencies and datasets into a Docker image. The repository submitted with
this paper contains instructions on how to use this Docker image and
reproduce all plots and tables in this paper.

\section{Ethics Statement}

FL trains models on private client datasets in a privacy-preserving manner.
However, FL does not completely eliminate privacy concerns, because the
transmitted model updates and the learned model parameters may expose the
private client data from which they are derived. Our work does not directly
target privacy concerns in FL. With that said, it is worth noting that
DAdaQuant does not expose any client data that is not already exposed
through standard FL training algorithms. In fact, DAdaQuant reduces the
amount of exposed data through lossy compression of the model
updates. We therefore believe that DAdaQuant is free of ethical
complications. 

\bibliography{paper} 
\bibliographystyle{iclr2022_conference}

\newpage
\clearpage
\appendix

\section{Additional simulation details and experiments}

\subsection{Additional simulation details}
\label{sec:models_datasets_detailed}

Here, we give detailed information on the models, datasets, training
objectives and implementation that we use for our experiments. We set the
five following FL tasks:

\begin{itemize}[noitemsep,leftmargin=*,topsep=0pt]
  \item Multinomial logistic regression (MLR) on a synthetic dataset called
  \dataset{Synthetic} that contains vectors in $\R^{60}$ with a label of one
  out of 10 classes.  We use the synthetic dataset generator in \citet{FedProx} to
  generate synthetic datasets. The generator samples \dataset{Synthetic}'s
  local datasets and labels from MLR models with randomly initialized
  parameters. For this purpose, parameters $\alpha$ and $\beta$ control
  different kinds of data heterogeneity. $\alpha$ controls the variation in
  the local models from which the local dataset labels are generated.
  $\beta$ controls the variation in the local dataset samples. We set
  $\alpha = 1$ and $\beta = 1$ to simulate an FL setting with both kinds of
  data heterogeneity. This makes \dataset{Synthetic} a useful testbed for
  FL.
  \item Character classification into 62 classes of handwritten characters
  from the \dataset{FEMNIST} dataset using a CNN. \dataset{FEMNIST} groups
  samples from the same author into the same local dataset.
  \item Smile detection in facial images from the \dataset{CelebA} dataset
  using a CNN. \dataset{CelebA} groups samples of the same person into the
  same local dataset. We note that LEAF's CNN for CelebA uses BatchNorm
  layers. We replace them with LayerNorm layers because they are more
  amenable to quantization. This change does not affect the final accuracy.
  \item Binary sentiment analysis of tweets from the \dataset{Sent140}
  dataset using an LSTM. \dataset{Sent140} groups tweets from the same user
  into the same local dataset. The majority of local datasets in the raw
  \dataset{Sent140} dataset only have a single sample. This impedes FL
  convergence. Therefore, we filter \dataset{Sent140} to clients with at
  least 10 samples (i.e. one complete batch). \Citet{LEAF, FedProx}
  similarly filter \dataset{Sent140} for their FL experiments.
  \item Next character prediction on text snippets from the
  \dataset{Shakespeare} dataset of Shakespeare's collected plays using an
  LSTM. \dataset{Shakespeare} groups lines from the same character into the
  same local dataset.
\end{itemize}
\Cref{tab:modelsanddatasets} provides statistics of our models and datasets. 
  
For our experiments in \Cref{fig:adaquantfl}, AdaQuantFL requires a
hyperparameter $s$ that determines the initial quantization level. We set
$s$ to 2, the optimal value reported by the authors of AdaQuantFL. The remaining
hyperparameters are identical to those used for the \dataset{Synthetic}
dataset experiments in \Cref{tab:results}.

We implement the models with PyTorch \citep{pytorch} and use Flower
\citep{Flower} to simulate the FL server and clients.

\begin{table}[h]
  \centering
  \small
  \begin{tabular}{lllllllll}
      \textbf{Dataset} & \textbf{Model} & \textbf{Parameters} & \textbf{Clients} & \textbf{Samples} & \multicolumn{4}{c}{\textbf{Samples per client}}  \\ \cline{6-9} 
      & & &                  &                  & \textbf{mean}         & \textbf{min} & \textbf{max} & \textbf{stddev}      \\ \hline
      Synthetic  & MLR & 610 &  30  & 9,600              &  320.0            & 45      & 5,953  &   1051.6                   \\
      FEMNIST          & 2-layer CNN  & $6.6\times10^6$  & 3,500 &  785,582           &  224.1          & 19    &  584 & 87.8                            \\
      CelebA          & 4-layer CNN  & $6.3\times10^5$  & 9,343 &  200,288           &  21.4          & 5    &  35 & 7.6                            \\
      Sent140          & 2-layer LSTM & $1.1\times10^6$ & 21,876 &    430,707           &   51.1         & 10   & 549 & 17.1         \\                   
      Shakespeare          & 2-layer LSTM & $1.3\times10^5$ & 1,129 &    4,226,158           &   3743         & 3   & 66,903 & 6212                            
      \end{tabular}
      \caption{Statistics of the models and datasets used for evaluation. MLR stands for ``Multinomial Logistic Regression''.} 
      \label{tab:modelsanddatasets}
\vspace*{-1.5em}  
\end{table}
  
\subsection{Complete communication-accuracy trade-off curves}
\label{sec:paretofull}

\begin{figure}[H]
  \begin{center}
    \hspace*{-0.5cm}
    \scalebox{0.75}{
      \input{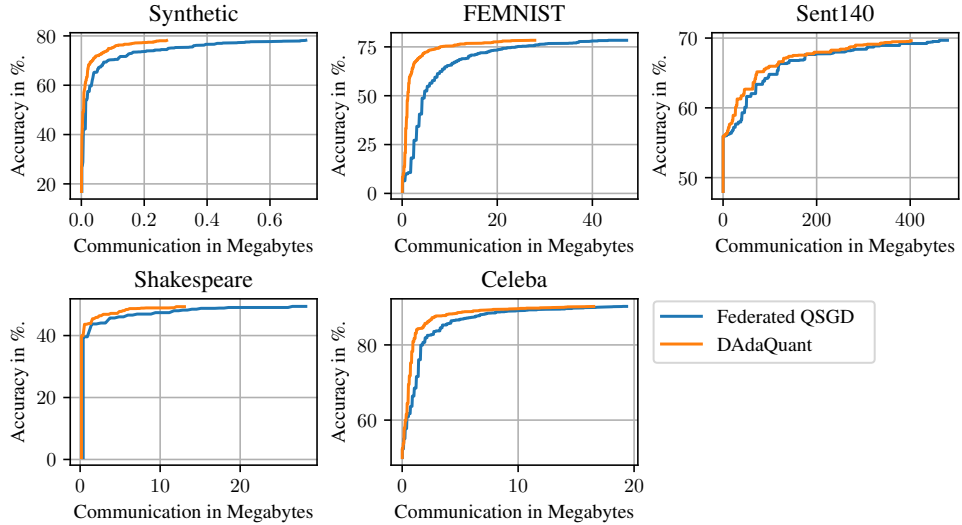}}
  \end{center}
  \vspace*{-0.3cm}
  \caption{Communication-accuracy trade-off curves for QSGD and DAdaQuant. We plot the average highest accuracies achieved up to any given amount of communication.}
\end{figure}

\subsection{Additional UVeQFed experiments}
\label{sec:uveqfed}

To demonstrate that the choice of UVeQFed's ``coding rate'' hyperparameter $R$
does not affect our findings on the superior compression factors of
DAdaQuant, we re-evaluate UVeQFed with $R=1$, which maximizes UVeQFed's compression factor.
Our results in \Cref{tab:uveqfed} show that with the exception of
\dataset{Shakespeare}, DAdaQuant still achieves considerably higher
compression factors than UVeQFed.

\begin{table}[H]
  \scriptsize
  \centering
  \begin{tabular}{l|l|l|l|l|l}
   & \multicolumn{1}{c|}{\textbf{Synthetic}} & \multicolumn{1}{c|}{\textbf{FEMNIST}} & \multicolumn{1}{c|}{\textbf{Sent140}} & \multicolumn{1}{c|}{\textbf{Shakespeare}} & \multicolumn{1}{c}{\textbf{Celeba}} \\ \hline
  \rowcolor[HTML]{EFEFEF} 
  \textbf{Uncompressed} & {$12.2$\,MB} & {$132.1$\,GB} & {$43.9$\,GB} & \cellcolor[HTML]{EFEFEF}{$267.0$\,MB} & \cellcolor[HTML]{EFEFEF}{$12.6$\,GB} \\
  \textbf{QSGD} & {$17\times$} & {$2809\times$} & {$90\times$} & {$9.5\times$} & {$648\times$} \\
  \rowcolor[HTML]{EFEFEF} 
  \textbf{UVeQFed (R=4)} & {$0.6\times$ ($0.03\!\times\!\!\!\!\!\times$)} & {$12\times$ ($0.00\!\times\!\!\!\!\!\times$)} & {$15\times$ ($0.16\!\times\!\!\!\!\!\times$)} & {$7.9\times$ ($0.83\!\times\!\!\!\!\!\times$)} & {$31\times$ ($0.05\!\times\!\!\!\!\!\times$)} \\
  \textbf{UVeQFed (R=1)} & {$13\times$ ($0.77\!\times\!\!\!\!\!\times$)} & {$34\times$ ($0.01\!\times\!\!\!\!\!\times$)} & {$41\times$ ($0.45\!\times\!\!\!\!\!\times$)} & {$\bf{21\times}$ ($\bf{2.22\!\times\!\!\!\!\!\times}$)} & {$93\times$ ($0.14\!\times\!\!\!\!\!\times$)} \\
  \rowcolor[HTML]{EFEFEF} 
  \textbf{DAdaQuant} & {$\bf{48\times}$ ($\bf{2.81\!\times\!\!\!\!\!\times}$)} & {$\bf{4772\times}$ ($\bf{1.70\!\times\!\!\!\!\!\times}$)} & {$\bf{108\times}$ ($\bf{1.19\!\times\!\!\!\!\!\times}$)} & \cellcolor[HTML]{EFEFEF}{$21\times$ ($2.21\!\times\!\!\!\!\!\times$)} & \cellcolor[HTML]{EFEFEF}{$\bf{775\times}$ ($\bf{1.20\!\times\!\!\!\!\!\times}$)}
  \end{tabular}
  \caption{Comparison of the compression factors of DAdaQuant, UVeQFed with
  $R=4$ (default value used for our experiments in \Cref{tab:results}) and
  UVeQFed with $R=1$ (maximizes UVeQFed's compression factor). Entry $p\!\times (q\!\times\!\!\!\!\!\times)$ denotes a compression factor of $p$ w.r.t. the uncompressed
  communication and a compression factor of $q$ w.r.t. Federated QSGD.}
  \label{tab:uveqfed}
  \end{table}

\section{Proofs}

\label{sec:proofs}

\begin{lemma}{For arbitrary $t>0$ and parameter $p_i \in [-t, t]$, let $s_i = \frac{t}{q_i}$, $b_i = \mathrm{rem}\brace{p_i, s_i}$ and $u_i = s_i - b_i$. Then, $\mathrm{Var}\brace{\quantizer_{q_i}(p_i)} = u_ib_i$}. 
\label{lemma:variance}
\end{lemma}
\begin{proof}
\begin{align*}
&\mathrm{Var}\brace{\quantizer_{q_i}(p_i)} & & \\
=~ &\mathrm{E}\bracket{\brace{\quantizer_{q_i}(p_i) - \mathrm{E}\bracket{\quantizer_{q_i}(p_i)}}^2} & & \\
=~ &\mathrm{E}\bracket{\brace{\quantizer_{q_i}(p_i) - p_i}^2} & & \quantizer_{q_i}(p_i) \text{ is an unbiased estimator of } p_i \\
=~ &\frac{b_i}{s_i}u_i^2 + \frac{u_i}{s_i}b_i^2 & & \textrm{cf. \cref{fig:stochastic_rounding}} \\
=~ &\frac{u_ib_i}{s_i}\brace{u_i+b_i} \\
=~ &u_ib_i
\qedhere
\end{align*}
\end{proof}

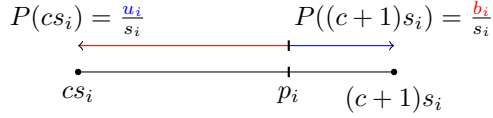
\begin{figure}[h!]
  \centering
  \begin{tikzpicture}[scale=0.7]
      \draw [-, color=black] (0,0) -- (4,0) -- (6, 0);
      \draw [thick] (4,-.1) node[below]{$p_i$} -- (4,0.1);
      \draw [fill] (0,0) circle [radius=.05];
      \node at (0,-.1) [below] {$cs_i$};
      \draw [fill] (6,0) circle [radius=.05];
      \node at (6,-.1) [below] {$(c+1)s_i$};
      \begin{scope}[shift={(0,-.5)}]
      \draw [{<[color=black]}-, color=red] (0,1) node[above,color=black]{$P(cs_i) = \frac{\color{blue}u_i}{\color{black}s_i}$} -- (4,1);
      \draw [-{>[color=black]}, color=blue] (4,1) -- (6,1) node[above,color=black]{$P((c+1)s_i) = \frac{\color{red}b_i}{\color{black}s_i}$};
      \draw [thick] (4,0.9) -- (4,1.1);
      \end{scope}
  \end{tikzpicture}
  \caption{Illustration of the Bernoulli random variable $\quantizer_{q_i}(p_i)$ in Lemma \ref{lemma:variance}. $s_i$ is the
  length of the quantization interval. $p_i$ is rounded up to $(c+1)s_i$ with a probability proportional
  to its distance from $cs_i$.}
  \label{fig:stochastic_rounding}
\end{figure}  

\begin{lemma}{Let $\quantizer$ be a fixed-point quantizer. Assume that
  parameters $p_1\ldots p_K$ are sampled from $\mathcal{U}[-t, t]$ for arbitrary $t >
  0$. Then, $\E_{p_1\ldots p_K}[\Var(e^{q_1\ldots q_K}_p)] = \frac{t^2}{6}\sum_{i=1}^{K} \frac{w_i^2}{q_i^2}$}. 
  \label{lemma:varianceq}
\end{lemma}
\begin{proof}
\begin{align*}
&\E_{p_1\ldots p_K}[\Var(e_p)] & & \\
=~ &\frac{1}{2t}\int_{-t}^t \frac{1}{2t}\int_{-t}^t \ldots \frac{1}{2t}\int_{-t}^t \mathrm{Var}\brace{\sum_{i=1}^{K}w_i \quantizer_{q_i}(p_i) - p}\,dp_1 dp_2 \ldots dp_K & & \\
=~ &\frac{1}{t}\int_0^t \frac{1}{t}\int_0^t \ldots \frac{1}{t}\int_0^t \mathrm{Var}\brace{\sum_{i=1}^{K}w_i \quantizer_{q_i}(p_i) - p}\,dp_1 dp_2 \ldots dp_K & & \textrm{\parbox[t]{3.5cm}{symmetry of $\quantizer_{q_i}(p_i)$ w.r.t. negation}}\\
=~ &\frac{1}{t^n}\int_0^t \int_0^t \ldots \int_0^t \sum_{i=1}^{K}w_i^2\mathrm{Var}\brace{\quantizer_{q_i}(p_i)}\,dp_1 dp_2 \ldots dp_K & & \text{\parbox[t]{3.5cm}{mutual independence of $ \quantizer_{q_i}(p_i)~\forall i$}} \\
=~ &\frac{1}{t^n}\sum_{i=1}^{K}\int_0^t \int_0^t \ldots \int_0^t w_i^2\mathrm{Var}\brace{\quantizer_{q_i}(p_i)}\,dp_1 dp_2 \ldots dp_K & & \text{\parbox[t]{3.5cm}{exchangeability of finite sums and integrals}} \\
=~ &\frac{1}{t^n}\sum_{i=1}^{K}t^{n-1}\int_0^t w_i^2\mathrm{Var}\brace{\quantizer_{q_i}(p_i)}\,dp_i & & \\
=~ &\frac{1}{t}\sum_{i=1}^{K} w_i^2\int_0^t\mathrm{Var}\brace{\quantizer_{q_i}(p_i)}\,dp_i & & \\
=~ &\frac{1}{t}\sum_{i=1}^{K} w_i^2\int_0^t u_i b_i\,dp_i & & \text{Lemma \ref{lemma:variance}} \\
=~ &\frac{1}{t}\sum_{i=1}^{K} w_i^2 q_i \int_0^{s_i} u_i b_i\,dp_i & & s_i\text{-periodicity of } u_i \text{ and } b_i \\
=~ &\frac{1}{t}\sum_{i=1}^{K} w_i^2 q_i \int_0^{s_i} \brace{s_i - p_i} p_i\,dp_i & & \\
=~ &\frac{1}{6t}\sum_{i=1}^{K} w_i^2 q_i s_i^3 & & \\
=~ &\frac{t^2}{6}\sum_{i=1}^{K} \frac{w_i^2}{q_i^2} & & 
\end{align*}
\end{proof}

\begin{lemma}{Let $\quantizer$ be a fixed-point quantizer. Assume that
parameters $p_1\ldots p_K$ are sampled from $\mathcal{U}[-t, t]$ for arbitrary $t >
0$. Then, $\min_{q_1\ldots q_K}\E_{p_1\ldots p_K}[\Var(e^{q_1\ldots q_K}_p)]$
subject to $Q = \sum_{i=1}^{K}q_i$ is minimized by $q_i =
Q\frac{w_i^{2/3}}{\sum_{k=1}^{K}{w_k^{2/3}}}$}. 
\label{lem:minimumq}
\end{lemma}
\begin{proof}
Define 

\begin{align*}
&f(\vq) = \E_{p_1\ldots p_K}[\Var(e^{q_1\ldots q_K}_p)] \\
&g(\vq) = \brace{\sum_{i=1}^n q_i} \\
&\Lagr\brace{\vq} = f(\vq) - \lambda g(\vq) \text{ (Lagrangian)}
\end{align*}

Any (local) minimum $\hat{\vq}$ satisfies
\begin{align*}
&\nabla \Lagr\brace{\hat{\vq}} = \mathbf{0} & & \\
\iff~ & \nabla \frac{t^2}{6} \sum_{i=1}^K \frac{w_i^2}{q_i^2} - \lambda \nabla \sum_{i=1}^K q_i = 0 \wedge \sum_{i=1}^K q_i = Q & & \text{Lemma \ref{lemma:varianceq}} \\
\iff~ & \forall i=1\ldots n.~ \frac{t^2}{-3} \frac{w_i^2}{q_i^3} = \lambda \wedge \sum_{i=1}^K q_i = Q & & \\
\iff~ & \forall i=1\ldots n.~ q_i = \sqrt[3]{\frac{t^2}{-3\lambda}w_i^2} \wedge \sum_{i=1}^K q_i = Q & & \\
\implies~ & \forall i=1\ldots n.~ q_i = Q\frac{w_i^{2/3}}{\sum_{j=1}^Kw_j^{2/3}}  & &
\end{align*}
\end{proof}

\subsection{Proof of \Cref{theorem:q}}

\begin{proof}
Using Lemma \ref{lem:minimumq}, it is straightforward to show that for any
$V$, $\min_{q_1\ldots q_K}\sum_{i=1}^{K}q_i$ subject to $\E_{p_1\ldots
p_K}[\Var(e^{q_1\ldots q_K}_p)] = V$ is minimized by $q_i = Cw_i^{2/3}$ for
the unique $C\in\R_{>0}$ that satisfies $\E_{p_1\ldots p_K}[\Var(e^{q_1\ldots q_K}_p)] = V$.

Then, taking $V = \E_{p_1\ldots p_K}[\Var(e^q_p)]$ and $C = \sqrt{\frac{a}{b}}$ (see \Cref{theorem:q}), we do indeed get
\begin{align*}
  &\E_{p_1\ldots p_K}[\Var(e^{q_1\ldots q_K}_p)] & & \\
  =~ & \frac{t^2}{6}\sum_{i=1}^{K} \frac{w_i^2}{{(Cw_i^{2/3})}^2} & & \text{Lemma \ref{lemma:varianceq}} \\
  =~ & \frac{1}{C^2}\frac{t^2}{6}\sum_{i=1}^{K} {w_i}^{2/3} & & \\
  =~ & \frac{\sum_{j=1}^K \frac{w_j^2}{q^2}}{\sum_{j=1}^K {w_j^{2/3}}}\frac{t^2}{6}\sum_{i=1}^{K} {w_i}^{2/3} \\
  =~ & \frac{t^2}{6}\sum_{j=1}^K \frac{w_j^2}{q^2} \\
  =~ & \E_{p_1\ldots p_K}[\Var(e^q_p)] & & \text{\cref{lemma:varianceq}}\
  \qedhere
  \end{align*}
\end{proof}

\end{document}

%% file: math_commands.tex
\usepackage{amsmath,amsfonts,bm}

\def\eqref#1{equation~\ref{#1}}

\def\1{\bm{1}}

\def\vp{{\bm{p}}}
\def\vq{{\bm{q}}}

\DeclareMathAlphabet{\mathsfit}{\encodingdefault}{\sfdefault}{m}{sl}
\SetMathAlphabet{\mathsfit}{bold}{\encodingdefault}{\sfdefault}{bx}{n}

\def\sC{{\mathbb{C}}}

\def\sN{{\mathbb{N}}}

\def\sS{{\mathbb{S}}}

\newcommand{\E}{\mathbb{E}}

\newcommand{\R}{\mathbb{R}}

\newcommand{\Var}{\mathrm{Var}}



%% file: samplequant_complex_static.pdf_tex
\begingroup%
  \makeatletter%
  \providecommand\color[2][]{%
    \errmessage{(Inkscape) Color is used for the text in Inkscape, but the package 'color.sty' is not loaded}%
    \renewcommand\color[2][]{}%
  }%
  \providecommand\transparent[1]{%
    \errmessage{(Inkscape) Transparency is used (non-zero) for the text in Inkscape, but the package 'transparent.sty' is not loaded}%
    \renewcommand\transparent[1]{}%
  }%
  \providecommand\rotatebox[2]{#2}%
  \newcommand*\fsize{\dimexpr\f@size pt\relax}%
  \newcommand*\lineheight[1]{\fontsize{\fsize}{#1\fsize}\selectfont}%
  \ifx\svgwidth\undefined%
    \setlength{\unitlength}{458.35926819bp}%
    \ifx\svgscale\undefined%
      \relax%
    \else%
      \setlength{\unitlength}{\unitlength * \real{\svgscale}}%
    \fi%
  \else%
    \setlength{\unitlength}{\svgwidth}%
  \fi%
  \global\let\svgwidth\undefined%
  \global\let\svgscale\undefined%
  \makeatother%
  \begin{picture}(1,0.58056508)%
    \lineheight{1}%
    \setlength\tabcolsep{0pt}%
    \put(0,0){\includegraphics[width=\unitlength,page=1]{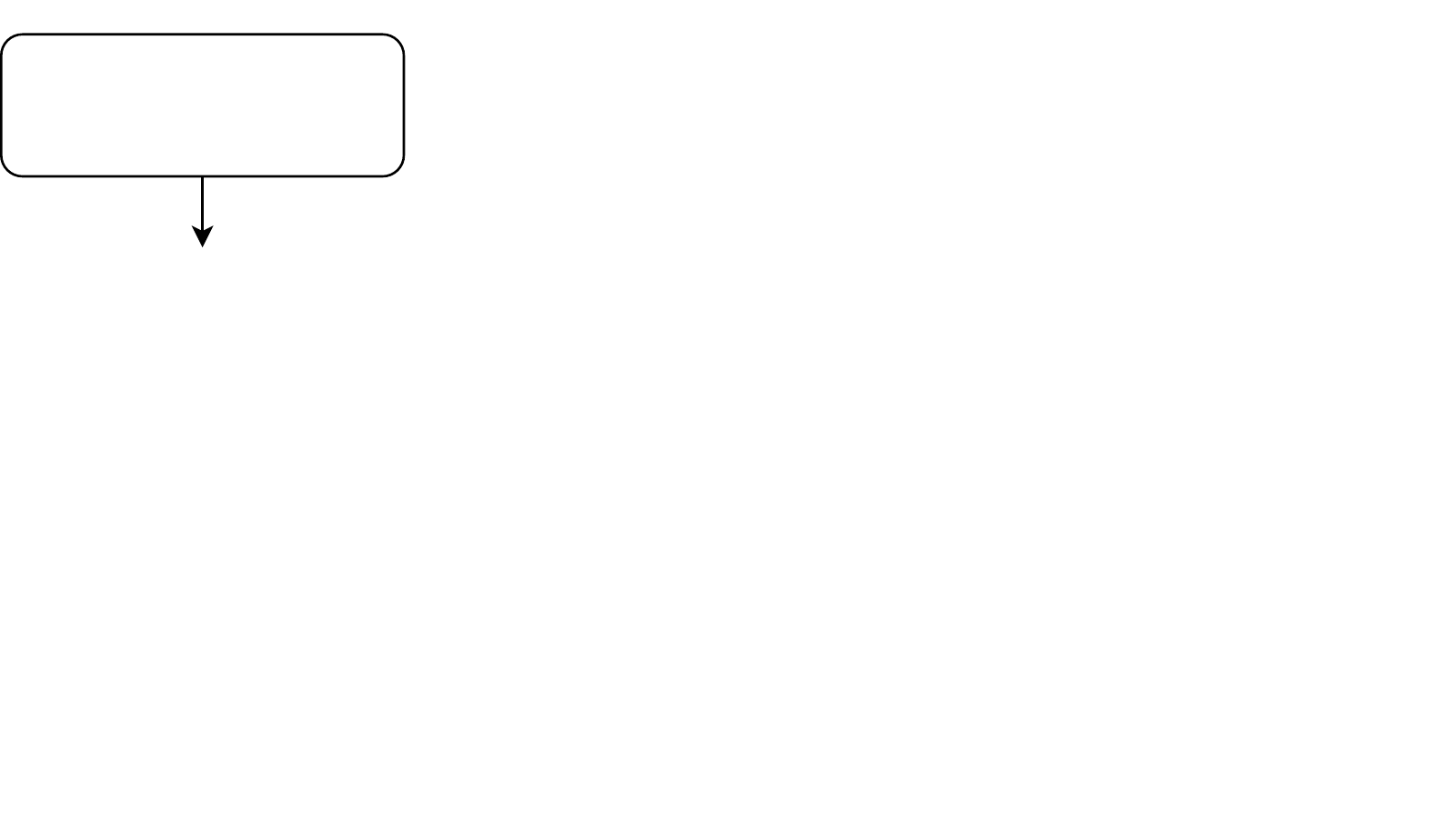}}%
    \put(0.10245747,0.49893155){\color[rgb]{0,0,0}\makebox(0,0)[t]{\lineheight{1.25}\smash{\begin{tabular}[t]{c}Client A\end{tabular}}}}%
    \put(0,0){\includegraphics[width=\unitlength,page=2]{anc/samplequant_complex_static.pdf}}%
    \put(0.82161821,0.50078706){\color[rgb]{0,0,0}\makebox(0,0)[t]{\lineheight{1.25}\smash{\begin{tabular}[t]{c}Client B\end{tabular}}}}%
    \put(0,0){\includegraphics[width=\unitlength,page=3]{anc/samplequant_complex_static.pdf}}%
    \put(0.48269995,0.33651424){\color[rgb]{0,0,0}\makebox(0,0)[t]{\lineheight{1.25}\smash{\begin{tabular}[t]{c}Server\end{tabular}}}}%
    \put(0,0){\includegraphics[width=\unitlength,page=4]{anc/samplequant_complex_static.pdf}}%
    \put(0.20871562,0.3881036){\color[rgb]{0,0,0}\makebox(0,0)[t]{\lineheight{1.25}\smash{\begin{tabular}[t]{c}\small$p_A \propto [0,1]$\end{tabular}}}}%
    \put(0,0){\includegraphics[width=\unitlength,page=5]{anc/samplequant_complex_static.pdf}}%
    \put(0.92776568,0.4038478){\color[rgb]{0,0,0}\makebox(0,0)[t]{\lineheight{1.25}\smash{\begin{tabular}[t]{c}\small$p_B \propto [0,1]$\end{tabular}}}}%
    \put(0,0){\includegraphics[width=\unitlength,page=6]{anc/samplequant_complex_static.pdf}}%
    \put(0.85904229,0.2287668){\color[rgb]{0,0,0}\makebox(0,0)[t]{\lineheight{1.25}\smash{\begin{tabular}[t]{c}$\quantizer(p_B)$\end{tabular}}}}%
    \put(0,0){\includegraphics[width=\unitlength,page=7]{anc/samplequant_complex_static.pdf}}%
    \put(0.48269995,0.06023088){\color[rgb]{0,0,0}\makebox(0,0)[t]{\lineheight{1.25}\smash{\begin{tabular}[t]{c}\small$\quantizer(p)$\end{tabular}}}}%
    \put(0.45674991,0.00428051){\color[rgb]{0,0,0}\makebox(0,0)[t]{\lineheight{1.25}\smash{\begin{tabular}[t]{c}\small$\mathbb{E}_{p_A, p_B}\left[\mathrm{Var}\left(\quantizer(p)\right)\right] = \mathbf{0.0018}$\end{tabular}}}}%
    \put(0,0){\includegraphics[width=\unitlength,page=8]{anc/samplequant_complex_static.pdf}}%
    \put(0.13908303,0.22713053){\color[rgb]{0,0,0}\makebox(0,0)[t]{\lineheight{1.25}\smash{\begin{tabular}[t]{c}\small$\quantizer(p_A)$\end{tabular}}}}%
    \put(0,0){\includegraphics[width=\unitlength,page=9]{anc/samplequant_complex_static.pdf}}%
    \put(0.14022672,0.32057146){\color[rgb]{0,0,0}\makebox(0,0)[t]{\lineheight{1.25}\smash{\begin{tabular}[t]{c}\small $q_A=8$ \end{tabular}}}}%
    \put(0,0){\includegraphics[width=\unitlength,page=10]{anc/samplequant_complex_static.pdf}}%
    \put(0.86018684,0.33046753){\color[rgb]{0,0,0}\makebox(0,0)[t]{\lineheight{1.25}\smash{\begin{tabular}[t]{c}\small $q_B=8$ \end{tabular}}}}%
    \put(0,0){\includegraphics[width=\unitlength,page=11]{anc/samplequant_complex_static.pdf}}%
    \put(0.28607995,0.22459151){\color[rgb]{0,0,0}\makebox(0,0)[t]{\lineheight{1.25}\smash{\begin{tabular}[t]{c}$\times \frac{1}{5}$\end{tabular}}}}%
    \put(0.68756631,0.2263409){\color[rgb]{0,0,0}\makebox(0,0)[t]{\lineheight{1.25}\smash{\begin{tabular}[t]{c}\small$\times \frac{4}{5}$\end{tabular}}}}%
  \end{picture}%
\endgroup%

%% file: samplequant_complex_dynamic.pdf_tex
\begingroup%
  \makeatletter%
  \providecommand\color[2][]{%
    \errmessage{(Inkscape) Color is used for the text in Inkscape, but the package 'color.sty' is not loaded}%
    \renewcommand\color[2][]{}%
  }%
  \providecommand\transparent[1]{%
    \errmessage{(Inkscape) Transparency is used (non-zero) for the text in Inkscape, but the package 'transparent.sty' is not loaded}%
    \renewcommand\transparent[1]{}%
  }%
  \providecommand\rotatebox[2]{#2}%
  \newcommand*\fsize{\dimexpr\f@size pt\relax}%
  \newcommand*\lineheight[1]{\fontsize{\fsize}{#1\fsize}\selectfont}%
  \ifx\svgwidth\undefined%
    \setlength{\unitlength}{457.875bp}%
    \ifx\svgscale\undefined%
      \relax%
    \else%
      \setlength{\unitlength}{\unitlength * \real{\svgscale}}%
    \fi%
  \else%
    \setlength{\unitlength}{\svgwidth}%
  \fi%
  \global\let\svgwidth\undefined%
  \global\let\svgscale\undefined%
  \makeatother%
  \begin{picture}(1,0.58117906)%
    \lineheight{1}%
    \setlength\tabcolsep{0pt}%
    \put(0,0){\includegraphics[width=\unitlength,page=1]{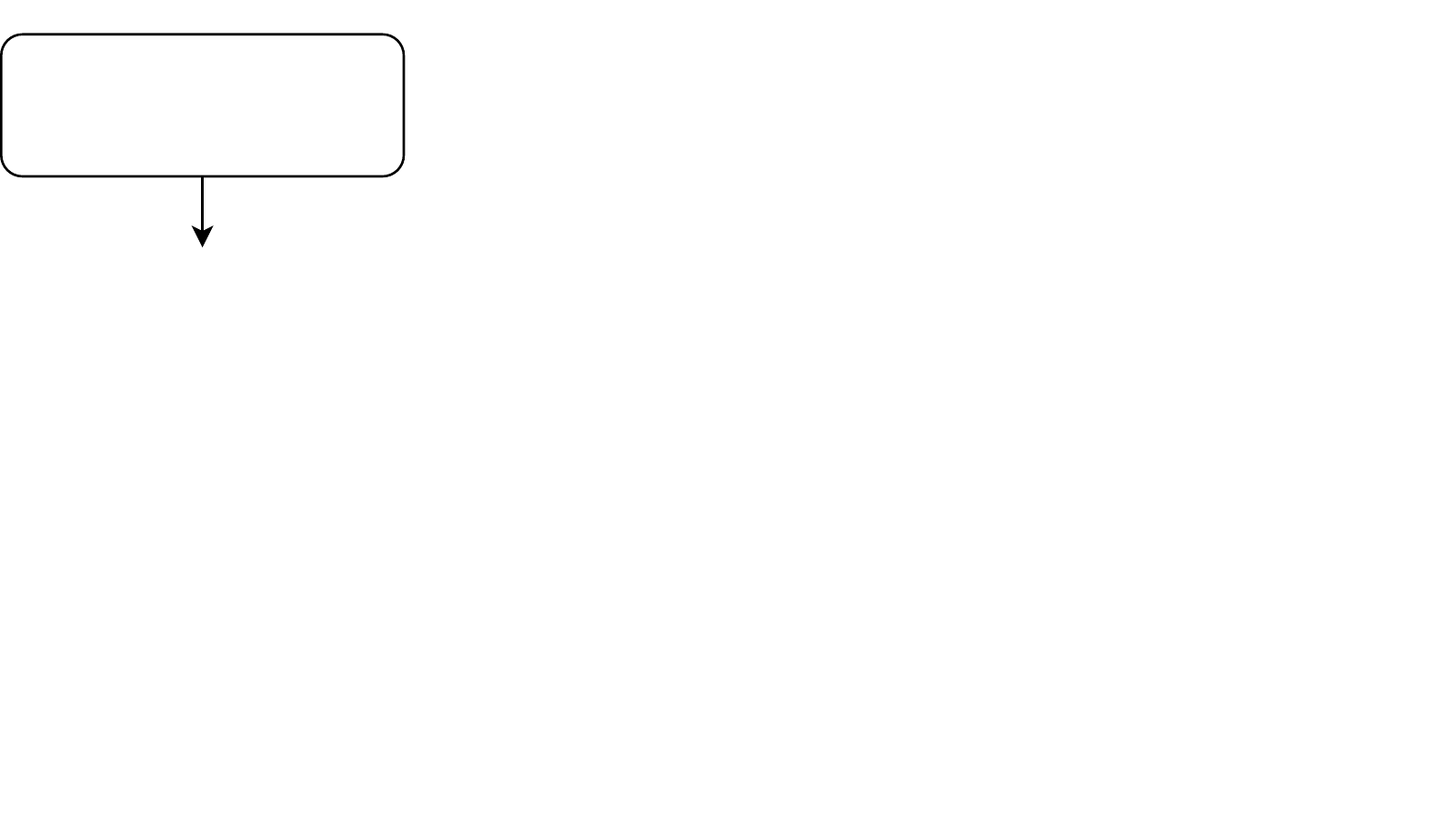}}%
    \put(0.10256583,0.49945918){\color[rgb]{0,0,0}\makebox(0,0)[t]{\lineheight{1.25}\smash{\begin{tabular}[t]{c}Client A\end{tabular}}}}%
    \put(0,0){\includegraphics[width=\unitlength,page=2]{anc/samplequant_complex_dynamic.pdf}}%
    \put(0.82248717,0.50131666){\color[rgb]{0,0,0}\makebox(0,0)[t]{\lineheight{1.25}\smash{\begin{tabular}[t]{c}Client B\end{tabular}}}}%
    \put(0,0){\includegraphics[width=\unitlength,page=3]{anc/samplequant_complex_dynamic.pdf}}%
    \put(0.48321048,0.3368701){\color[rgb]{0,0,0}\makebox(0,0)[t]{\lineheight{1.25}\smash{\begin{tabular}[t]{c}Server\end{tabular}}}}%
    \put(0,0){\includegraphics[width=\unitlength,page=4]{anc/samplequant_complex_dynamic.pdf}}%
    \put(0.20893637,0.38851402){\color[rgb]{0,0,0}\makebox(0,0)[t]{\lineheight{1.25}\smash{\begin{tabular}[t]{c}\small$p_A \propto [0,1]$\end{tabular}}}}%
    \put(0,0){\includegraphics[width=\unitlength,page=5]{anc/samplequant_complex_dynamic.pdf}}%
    \put(0.92874693,0.40427488){\color[rgb]{0,0,0}\makebox(0,0)[t]{\lineheight{1.25}\smash{\begin{tabular}[t]{c}\small$p_B \propto [0,1]$\end{tabular}}}}%
    \put(0,0){\includegraphics[width=\unitlength,page=6]{anc/samplequant_complex_dynamic.pdf}}%
    \put(0.85995086,0.22900871){\color[rgb]{0,0,0}\makebox(0,0)[t]{\lineheight{1.25}\smash{\begin{tabular}[t]{c}$\quantizer(p_B)$\end{tabular}}}}%
    \put(0,0){\includegraphics[width=\unitlength,page=7]{anc/samplequant_complex_dynamic.pdf}}%
    \put(0.48321048,0.06029454){\color[rgb]{0,0,0}\makebox(0,0)[t]{\lineheight{1.25}\smash{\begin{tabular}[t]{c}\small$\quantizer(p)$\end{tabular}}}}%
    \put(0,0){\includegraphics[width=\unitlength,page=8]{anc/samplequant_complex_dynamic.pdf}}%
    \put(0.45723301,0.00428499){\color[rgb]{0,0,0}\makebox(0,0)[t]{\lineheight{1.25}\smash{\begin{tabular}[t]{c}\small$\mathbb{E}_{p_A, p_B}\left[\mathrm{Var}\left(\quantizer(p)\right)\right] = \mathbf{0.0017}$\end{tabular}}}}%
    \put(0,0){\includegraphics[width=\unitlength,page=9]{anc/samplequant_complex_dynamic.pdf}}%
    \put(0.13923014,0.2273707){\color[rgb]{0,0,0}\makebox(0,0)[t]{\lineheight{1.25}\smash{\begin{tabular}[t]{c}\small$\quantizer(p_A)$\end{tabular}}}}%
    \put(0,0){\includegraphics[width=\unitlength,page=10]{anc/samplequant_complex_dynamic.pdf}}%
    \put(0.14037503,0.32091045){\color[rgb]{0,0,0}\makebox(0,0)[t]{\lineheight{1.25}\smash{\begin{tabular}[t]{c}\small $q_A=4$ \end{tabular}}}}%
    \put(0,0){\includegraphics[width=\unitlength,page=11]{anc/samplequant_complex_dynamic.pdf}}%
    \put(0.86109663,0.33081699){\color[rgb]{0,0,0}\makebox(0,0)[t]{\lineheight{1.25}\smash{\begin{tabular}[t]{c}\small $q_B=9$ \end{tabular}}}}%
    \put(0,0){\includegraphics[width=\unitlength,page=12]{anc/samplequant_complex_dynamic.pdf}}%
    \put(0.28638251,0.22482899){\color[rgb]{0,0,0}\makebox(0,0)[t]{\lineheight{1.25}\smash{\begin{tabular}[t]{c}$\times \frac{1}{5}$\end{tabular}}}}%
    \put(0.68829352,0.22658024){\color[rgb]{0,0,0}\makebox(0,0)[t]{\lineheight{1.25}\smash{\begin{tabular}[t]{c}\small$\times \frac{4}{5}$\end{tabular}}}}%
  \end{picture}%
\endgroup%